\documentclass[sigconf,nonacm]{aamas}
\usepackage{balance} 

\usepackage{times}
\usepackage{soul}
\usepackage{url}
\usepackage[utf8]{inputenc}
\usepackage{caption}
\usepackage{graphicx}
\usepackage{amsmath}
\usepackage{amsthm}
\usepackage{booktabs}
\usepackage{algorithm}
\usepackage{algorithmic}
\usepackage[switch]{lineno}
\usepackage{amsfonts}
\usepackage[flushleft]{threeparttable}
\usepackage{newfloat}
\usepackage{listings}
\usepackage{multirow}
\usepackage{booktabs}
\usepackage{enumitem}
\usepackage{soul}

\usepackage{tikz}
\usepackage{pgfplots}
\usetikzlibrary{arrows.meta,shapes}
\usetikzlibrary{pgfplots.groupplots}
\usetikzlibrary{patterns}
\setul{1pt}{.4pt}

\newtheorem{thm}{Theorem}
\newtheorem{lemma}[thm]{Lemma}

\setcopyright{ifaamas}

\title{Clique Analysis and Bypassing in Continuous-Time Conflict-Based Search}

\author{Thayne T. Walker}
\affiliation{
  \institution{University of Denver, Lockheed Martin Corporation}
  \city{Denver}
  \country{USA}}
\email{thayne.walker@du.edu}

\author{Nathan R. Sturtevant}
\affiliation{
  \institution{Department of Computing Science, Alberta Machine Intelligence Institute (Amii), University of Alberta}
  \city{Edmonton}
  \country{Canada}}
\email{nathanst@ualberta.ca}

\author{Ariel Felner}
\affiliation{
  \institution{Ben Gurion University}
  \city{Be'er-Sheva}
  \country{Israel}}
\email{felner@bgu.ac.il}

\begin{abstract}
While the study of unit-cost Multi-Agent Pathfinding (MAPF) problems has been popular, many real-world problems require continuous time and costs due to various movement models.
In this context, this paper studies symmetry-breaking enhancements for Continuous-Time Conflict-Based Search (CCBS), a solver for continuous-time MAPF.
Resolving conflict symmetries in MAPF can require an exponential amount of work. We adapt known enhancements from unit-cost domains for CCBS: \textit{bypassing}, which resolves cost symmetries and \textit{biclique constraints} which resolve spatial conflict symmetries. We formulate a novel combination of biclique constraints with disjoint splitting for spatial conflict symmetries. Finally, we show empirically that these enhancements yield a statistically significant performance improvement versus previous state of the art, solving problems for up to 10\% or 20\% more agents in the same amount of time on dense graphs.
\end{abstract}

\begin{document}


\pagestyle{fancy}
\fancyhead{}


\maketitle

\section{Introduction} The objective of multi-agent pathfinding (MAPF) is to find paths for multiple agents such that each agent reaches its goal without conflicting with other agents. Agents paths are conflicting if at any time their shapes overlap. MAPF has applications in warehouses \cite{li2021lifelong}, package delivery~\cite{choudhury2021efficient}, games~\cite{Silver05,botea2013pathfinding}, firefighting~\cite{roldan2021survey}, search and rescue~\cite{scherer2015autonomous} and intersection management~\cite{dresner2008multiagent,vsvancara2019online}.

A significant amount of prior work has focused on ``classic'' MAPF with a state space represented by a grid or planar graph with unit-cost edges~\cite{stern2019multi}. Therefore, in classic MAPF, all actions take one time step and agents always occupy exactly one vertex in a time step. These limitations simplify the problem, but cannot be applied to domains which may exhibit variable size agents and continuous-time, variable-duration motion and wait actions. We seek optimal solutions to the continuous-time MAPF problem (MAPF\textsubscript{R})~\cite{walker2018extended,andreychuk2019continuous}, denoted MAPF\textsubscript{R} for real-valued action durations and costs on general graphs (e.g., planar, non-panar, unit-cost and non-unit cost graphs).

Continuous-Time Conflict-Based Search (CCBS)~\cite{andreychuk2022multi} is a solver for MAPF\textsubscript{R}. CCBS re-formulates the Conflict-Based Search (CBS) algorithm~\cite{CBS} to allow variable-duration wait actions and constraints which account for continuous-time execution. CCBS was shown to be effective on a several settings that are inspired by real-world applications. Additional enhancements, such as heuristics~\cite{li2019improved}, disjoint splitting (DS)~\cite{li2019disjoint} and conflict prioritization~\cite{ICBS} were added to CCBS~\cite{andreychuk2021improving}. These enhancements improve the runtime of CCBS.
In contrast to other prior optimal continuous-time approaches~\cite{walker2018extended,walker2020generalized,walker2021conflict} which assume only fixed-duration wait actions, CCBS plans optimal, arbitrary duration wait times.

CCBS represents a significant advancement for MAPF\textsubscript{R}. However, conflict symmetries continue to pose a problem for this algorithm. A \emph{conflict symmetry}~\cite{li2019symmetry} occurs when two or more agents are situated such that one or both agents must increase their path cost to avoid conflict. For optimal algorithms like CCBS, this means that many lower-cost alternate paths must be explored  before proving that the cost increase is necessary. Conflict symmetries can cause an exponential amount of work to resolve~\cite{li2019symmetry}. In this paper, we address conflict symmetries by adapting and building new symmetry-breaking enhancements for CCBS:
\begin{itemize}[leftmargin=.35cm]
    \item We adapt the bypass (BP) enhancement~\cite{CBSBP}, originally formulated for CBS, to be used with CCBS.
    \item We adapt Biclique constraints (BC)~\cite{walker2020generalized}, originally formulated for CBS, to be used with CCBS.
    \item We combine disjoint splitting (DS)~\cite{li2019disjoint} as formulated for CCBS~\cite{andreychuk2021improving} with (BC): \emph{disjoint bicliques} (DB)
    \item We newly re-formulate BC for k-partite cliques and combine with DS: \emph{disjoint k-partite cliques} (DK)
\end{itemize}

This paper is organized as follows: We first provide a definition of the MAPF\textsubscript{R} problem. Next, we describe the CCBS algorithm and related work. This is followed by a description of the new symmetry breaking enhancements. Finally, we present a comprehensive ablation study on all of the enhancements.

\section{Problem Definition}
MAPF was originally defined for a ``classic" setting~\cite{stern2019multi} where the movements of agents are coordinated on a grid, which is a 4-neighbor planar graph. Edges have a unit cost/unit time duration and agents occupy a point in space. Thus, two agents can only have conflicts when on the same vertex at the same time, or traversing the same edge in opposite directions. MAPF\textsubscript{R}~\cite{walker2018extended,andreychuk2019continuous}, an extension of MAPF for \emph{real-valued} action durations, uses a weighted graph $G{=}(V,E)$ which may be non-planar. Every vertex $v{\in}V$ is associated with coordinates in a metric space and every edge $e{\in}E$ is associated with a positive real-valued edge weight $w(e){\in}\mathbb{R}_+$. For the purposes of this paper, weights represent the times it takes to traverse edges. However, time duration and cost can be treated separately. There are $k$ agents, $A{=}\{1,..,k\}$. Each agent has a start and a goal vertex $V_s{=}\{start_1,..,start_k\}{\subseteq}V$ and $V_g{=}\{goal_1,..,goal_k\}{\subseteq}V$ such that $start_i{\neq} start_j$ and $ goal_i{\neq}goal_j$ for all $i{\neq}j$.

A \emph{solution} to a MAPF\textsubscript{R} instance is $\Pi{=}\{\pi_1,..,\pi_k\}$, a set of single-agent \emph{paths} which are sequences of \emph{states}. A state $s{=}(v,t)$ is a pair composed of a vertex $v{\in}V$ and a time $t{\in}\mathbb{R}_+$.
A path for agent $i$ is a sequence of $d{+}1$ states $\pi_i{=}[s_i^0,..,s_i^d]$, where $s_i^0{=}(start_i,0)$ and $s_i^d{=}(goal_i,t_g)$ where $t_g$ is the time  the agent arrives at its goal and all vertices in the path follow edges in $E$.

Agents have a shape which is situated relative to an agent-specific \emph{reference point}~\cite{li2019large}. Agents' shapes may vary, but this paper uses circular agents with a radius of $\sqrt{\scriptstyle 2}/4$ units. Agents move along edges along a straight vector in the metric space. Traversing an edge is called an \emph{action}, $a{=}(s,s')$, where $s$ and $s'$ are a pair of neighboring states. \emph{Wait actions} of any positive, real-valued duration is allowed at any vertex. Variable velocity and acceleration are allowed. For simplicity, this paper assumes fixed velocity motion with no acceleration.

A \emph{conflict} happens when two agents perform actions $\langle a_i, a_j\rangle$ such that their shapes overlap at the same time~\cite{walker2018extended}. A \emph{feasible solution} has no conflicts between any pairs of its constituent paths. The objective is to minimize the sum-of-costs $c(\Pi){=}\sum_{\pi\in\Pi} c(\pi)$, where $c(\pi)$ is the sum of edge weights of all edges traversed in $\pi$.
We seek $\Pi^*$, a solution with minimal cost among all feasible solutions. Optimization of the classic MAPF problem is NP-hard~\cite{YuLavalle2013}, hence, optimization of the MAPF\textsubscript{R} problem is also NP-hard.

\section{Background}
We now describe CCBS and other prior work.
\subsection{Conflict-Based Search}
Continuous-time Conflict-Based Search (CCBS)~\cite{andreychuk2022multi} is based on the classic Conflict-Based Search (CBS)~\cite{CBS} algorithm, so we describe it next. CBS performs search on two levels. The \emph{high level} searches a constraint tree ($CT$). Each node $N$ in the $CT$ contains a solution $N.\Pi$, and a set of constraints $N.C$. A \emph{constraint} blocks an agent from performing the action(s) that caused the conflict and is defined as a tuple $\langle i,v,t\rangle$, where $i$ is the agent, $v$ is the vertex and $t$ is the time. Each path $\pi_i{\in}N.\Pi$ of agent $i$ in the root node is constructed using a \emph{low-level} search without taking any other agents into account. Next, CBS checks for conflicts between any pairs of paths $\pi_i$ and $\pi_j$ in $N.\Pi$. If $N.\Pi$ contains no conflict, then $N$ is a goal node and CBS terminates. If $N.\Pi$ contains a conflict between any $\pi_i$ and $\pi_j$, then CBS performs a \emph{split}, meaning that it generates two child nodes $N_i$ and $N_j$ of $N$ and adds constraints $c_i$ and $c_j$ to $N_i.C$ and $N_j.C$ respectively. 

Next, CBS re-plans $\pi_i{\in}N_i.\Pi$ and $\pi_j{\in}N_j.\Pi$  with  constraints $c_i$ and $c_j$ and other constraints inherited from ancestor nodes so that the current conflict and previously detected conflicts are avoided. CBS searches the tree in a best-first fashion, prioritized by the sum-of-costs. CBS terminates when a feasible solution is found.

A significant amount of improvements for CBS have been proposed such as adding high-level heuristics~\cite{felner2018adding,li2019improved}, conflict prioritization~\cite{ICBS}, allowing multiple constraints per split~\cite{li2019large}, disjoint splitting~\cite{li2019disjoint} and conflict symmetry resolution~\cite{li2019symmetry,zhang2020multi,li2021pairwise}. Some enhancements were also proposed for MAPF\textsubscript{R}, such as kinodynamic constraints~\cite{cohen2019optimal,kottinger2022conflict,wen2022cl},  biclique constraints (BC)~\cite{walker2020generalized} (which will be described later) and CCBS itself~\cite{andreychuk2022multi}, which we describe next.

\begin{figure}[t!]
\begin{tikzpicture}[scale=.9]
\path[-,thin] (0,0)  edge node[yshift=5pt,xshift=-5pt] {\small 2} (1.73,1);
\path[-,thin] (0,0)  edge node[yshift=-6pt] {\small$\sqrt{3}$} (1.73,0);
\path[-,thin] (1.73,0)  edge node[xshift=5pt] {\small 1} (1.73,1);

\node[circle,draw=black,fill=black,inner sep=0pt,minimum size=2pt] at (0,0) {};
\node[circle,draw=black,fill=black,inner sep=0pt,minimum size=2pt] at (1.73,1) {};
\node[circle,draw=black,fill=black,inner sep=0pt,minimum size=2pt] at (1.73,0) {};
\node at (-.2,-.2) {A};
\node at (1.93,1.2) {B};
\node at (1.93,-.2) {C};
\node at (.9,-.9) {(a)};
\end{tikzpicture}
\begin{tikzpicture}[scale=.9]
\path[-,thin] (0,0)  edge (1.73,1);
\path[-,thin] (0,0)  edge (1.73,0);
\path[-,thin] (1.73,0)  edge (1.73,1);
\path[->,>=stealth,dashed,very thick,color=red,bend left] (0,.2)  edge (1.4,1.0);
\path[->,>=stealth,dashed,very thick,color=blue,bend left] (2.1,1) edge node {} (2.1,0);
\path[->,>=stealth,dashed,very thick,color=orange!80!black,bend left] (1.73,-.33) edge node {} (0,-.33);

\node[circle,draw=red,text=red,thick,fill=red!10,inner sep=0pt,minimum size=17pt] at (0,0) {$x$};
\node[circle,draw=blue,text=blue,thick,fill=blue!10,inner sep=0pt,minimum size=17pt] at (1.73,1) {$y$};
\node[circle,draw=orange!80!black,text=orange!60!black,thick,fill=orange!20,inner sep=0pt,minimum size=17pt] at (1.73,0) {$z$};
\node at (.9,-.9) {(a)};
\end{tikzpicture}
\begin{tikzpicture}[scale=.9]
\node[circle,draw=red,text=red,thick,fill=red!10,inner sep=0pt,minimum size=17pt] (A) at (0,0) {$x$};
\node[circle,draw=blue,text=blue,thick,fill=blue!10,inner sep=0pt,minimum size=17pt] (B) at (1.73,1) {$y$};
\node[circle,draw=orange!80!black,text=orange!60!black,thick,fill=orange!20,inner sep=0pt,minimum size=17pt] (C) at (1.73,0) {$z$};
\draw [->,>=stealth,dashed,very thick,color=blue] (B) edge[loop right]node[text=black]{\small .36} (B);
\draw [->,>=stealth,dashed,very thick,color=orange!80!black] (C) edge[loop right]node[text=black]{\small .36} (C);
\draw[->,>=stealth,dashed,very thick,color=red] (A) edge (B);
\path[->,>=stealth,dashed,very thick,color=blue] (B) edge(C);
\path[->,>=stealth,dashed,very thick,color=orange!80!black] (C) edge node {} (A);
\node at (1.2,-.9) {(b)};
\end{tikzpicture}
\caption{(a) A planning graph, (b) a MAPF\textsubscript{R} instance for agents $x$, $y$ and $z$, and (c) a solution that uses fractional wait times.}
\label{fig:hard-example}
\end{figure}
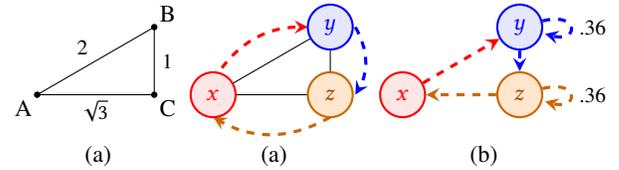

\subsection{Safe Interval Path Planning}

Safe Interval Path Planning (SIPP)~\cite{phillips2011sipp} is an algorithm for planning a single agent on the same graph as moving obstacles. In the case of CCBS, the moving obstacles are other agents. The usage of SIPP with CCBS will be discussed next. SIPP uses an A*-based algorithm with a specialized successor generation routine. During successor generation, SIPP computes the actions available to the agent at each graph vertex by removing actions which would result in conflicts and/or adding wait actions with a specific duration to avoid conflicts. SIPP does this by computing a set of \emph{safe intervals} for each vertex, that is, time intervals in which an agent may occupy a vertex without conflicting with moving obstacles. In this way, SIPP guarantees conflict-free shortest paths.

\subsection{Continuous-time Conflict-Based Search}

CCBS~\cite{andreychuk2022multi} modifies CBS by allowing continuous-time actions. This is accomplished by adding additional functionality to CBS:
\begin{itemize}[leftmargin=.35cm]
    \item CCBS uses continuous-time-and-space collision detection. It also computes exact wait times instead of using fixed duration wait actions~\cite{walker2019collision,guy2019guide,jimenez20013d}.
    \item CCBS handles \emph{durative} conflicts, (where agents' shapes overlap for a period of time), by utilizing time-range constraints~\cite{atzmon2018robust}.
    \item CCBS uses SIPP~\cite{phillips2011sipp} at the low level. In CCBS, SIPP is adapted to interpret time-range constraints~\cite{atzmon2018robust} (i.e., unsafe intervals) as safe intervals. Safe intervals are used to generate exact-duration wait actions to avoid conflicts.
\end{itemize}

We now describe CCBS in detail. An example for CCBS is illustrated in Figure \ref{fig:hard-example}. Figure \ref{fig:hard-example}(b) shows a problem instance on the simple planning graph in Figure \ref{fig:hard-example}(a) in which three agents exist, one at each vertex. Each of the agents needs to rotate one edge in the clockwise direction (or two edges in the counter-clockwise direction) in order to reach their goal. Assuming an agent radius of $\sqrt{\scriptstyle 2}/4$, actions $A{\rightarrow}B$ and $C{\rightarrow}A$ conflict if taken simultaneously.

After CCBS detects the conflict, \emph{time-range} constraints are constructed for the agents. Time-range constraints block agents from performing actions inside of a given time range. This is done by computing the delay time necessary for the action $C{\rightarrow}A$ to avoid conflict with the action $A{\rightarrow}B$ (and vice-versa). This delay time is used to create an \emph{unsafe interval}, the interval in which, if the action is taken, it will conflict. In this case, the unsafe interval for action $C{\rightarrow}A$ is $[0,0.36)$. Hence, if action $C{\rightarrow}A$ is delayed by 0.36, it can be executed without conflicting with action $A{\rightarrow}B$.

After CCBS constructs the time-range constraint for agent $z$, the low-level SIPP solver is called. CCBS re-interprets the unsafe interval $[0,0.36)$ to the safe interval, $[0.36,\infty)$ for SIPP. Thus when SIPP is called for agent $z$, it will be forced to wait 0.36 time steps before executing action $C{\rightarrow}A$, and the conflict with the action $A{\rightarrow}B$ is avoided. The wait action is shown by the self-loop on agent $z$ in the solution, Figure \ref{fig:hard-example}(c).

This problem instance is unsolvable without a non-unit cost wait action. If, for example, agent $z$ were to wait one full time step, it would cause agent $y$ to wait one full time step as well, causing it to conflict with agent $x$. This is one reason why arbitrary-duration wait actions are important for MAPF\textsubscript{R}. In addition, lower-cost solutions are possible with arbitrary-duration wait actions, since agents are allowed to wait for fractional times instead of whole time steps.

There are several enhancements for CCBS as described so far. The CCBS authors introduced a high-level heuristic based on the max-weight independent set problem. It was formerly formulated as as an integer linear program (ILP) for classic MAPF~\cite{li2019improved}, but reformulated for continuous time as a linear program (LP)~\cite{andreychuk2021improving}. Finally, a special formulation of disjoint splitting~\cite{andreychuk2021improving} was added to CCBS. Disjoint splitting is now explained in further detail.

\subsection{Disjoint Splitting}

Disjoint splitting (DS)~\cite{li2019disjoint} is a technique for CBS which helps avoid resolving the same conflict multiple times in different sub-trees of the CT. The procedure for DS is as follows: when performing a split, one child node uses a \emph{negative constraint} defined as a tuple $\langle i,v,t\rangle$ and causes the agent ($i$) to avoid a vertex ($v$) at a specific time ($t$). The other child node uses a \emph{positive constraint}, $\langle i,v,t\rangle$, which forces the agent ($i$) to pass through a vertex ($v$) at a specific time ($t$). A positive constraint for agent $i$ also acts as a negative constraint for all other agents so that they avoid conflicting with agent $i$.

Positive constraints are enforced at the low level by adding time-specific sub-goals or \emph{landmarks} to the search. The constraints in both child nodes are for the same agent, but there is a choice of which conflicting agent to split on. It was shown that disjoint splitting helps CBS to do less work in general~\cite{li2019disjoint}.

The implementation of DS for CCBS differs from CBS in two ways~\cite{andreychuk2021improving}: (1) Constraints for agent $i$ in CCBS include a time range $\langle i,v,t_0,t_2\rangle$~\cite{atzmon2018robust}. Since the arrival at a landmark location is allowed at multiple times, special logic is required to determine which exact time is optimal and feasible with respect to all other constraints. (2) Unlike DS for Classic MAPF, positive constraints do not act as a negative constraint for all other agents. Instead, a single negative constraint is added for agent $j$ to help it avoid the landmark location for agent $i$.

This second difference is a limitation that must be solved in CCBS using bipartite analysis, covered later in this paper. Despite this weakness, DS yields consistent improvements in runtime performance versus the original splitting method~\cite{andreychuk2021improving}.

\subsection{Biclique Constraints}
Biclique constraints was proposed for MAPF\textsubscript{R}~\cite{walker2020generalized} with unit-cost wait actions, but it has never been studied with variable-duration wait actions until now. Recall from the discussion of CBS that, during a split, agents $i$ and $j$ each receive a new constraint $c_i$ and  $c_j$ in their respective child nodes. For BC to work, CCBS must be combined with multi-constraint CBS (MCBS)~\cite{li2019large}, which allows each agent to receive \emph{sets} of one or more new constraints $C_i$ and $C_j$, respectively. Completeness is ensured only when the actions blocked by $C_i$ are mutually conflicting with all actions blocked by $C_j$. This is known as the mutually disjunctive property~\cite{li2019large,atzmon2019multi} of constraint sets. Constraint sets are \emph{valid} iff no solution exists when both agents $i$ and $j$ violate any constraints $c_i{\in}C_i$ and $c_j{\in}C_j$, respectively, simultaneously.
One approach to discovering valid sets of constraints for MAPF\textsubscript{R} settings involves analysis of bipartite conflict graphs~\cite{walker2020generalized}. See Figure \ref{fig:BC}.

\begin{figure}[t!]
\begin{tikzpicture}[scale=1]
\node (v) at (1.75,1.75) {\tiny $\;$};
\draw[step=1cm,color=gray,shift={(.5,0.25)}] ([shift={(-1.26,-1.51)}]v) grid (2,2);

\node[circle,draw=red,text=red,text=red,thick,fill=red!10,inner sep=0pt,minimum size=20pt] at ([shift={(-.75,0)}]v) {$x$};
\node[circle,draw=blue,text=blue,text=blue,thick,fill=blue!10,inner sep=0pt,minimum size=20pt] at ([shift={(-.75,-1)}]v) {$y$};

\path[->,>=stealth,dashed,very thick,color=blue,bend left] (1.2,1)  edge node {} (2,1.75);
\path[->,>=stealth,dashed,very thick,color=red,bend right] (1.2,1.5) edge node {} (2,.75);

\node at ([shift={(-.25,-1.8)}]v) {(a)};

\end{tikzpicture}
\begin{tikzpicture}[scale=1,
label/.style={rectangle,draw,black,fill=white,text=black,inner sep=1pt}]
\node (v) at (1.75,1.75) {\tiny $\;$};
\draw[step=1cm,color=gray,shift={(.5,0.25)}] ([shift={(-1.26,-1.51)}]v) grid (2,2);

\node[circle,draw=red,text=red,text=red,thick,fill=red!10,inner sep=0pt,minimum size=20pt] at ([shift={(-.75,0)}]v) {$x$};
\node[circle,draw=blue,text=blue,text=blue,thick,fill=blue!10,inner sep=0pt,minimum size=20pt] at ([shift={(-.75,-1)}]v) {$y$};

\path[->,>=stealth,very thick,color=blue] (1.25,1)  edge node {} (2,1.75);
\node at (1.6,1.42) [label] (l1) {\small 5 };
\path[->,>=stealth,very thick,color=red] (1.25,1.5) edge node {} (2,.75);
\node at (1.6,1.08) [label] (l1) {\small 2 };
\path[->,>=stealth,very thick,color=blue] (1.25,.75)  edge node {} (2,.75);
\node at (1.6,.75) [label] (l1) {\small 6 };
\path[->,>=stealth,very thick,color=red] (1.25,1.75) edge node {} (2,1.75);
\node at (1.6,1.75) [label] (l1) {\small 1 };
\path[->,>=stealth,very thick,color=blue] (1.1,1.1)  edge node {} (1.1,1.75);
\node at (1.15,1.3) [label] (l1) {\small 4 };
\path[->,>=stealth,very thick,color=red] (.9,1.4) edge node {} (.9,.75);
\node at (.85,1.2) [label] (l1) {\small 3 };


\node at ([shift={(-.25,-1.8)}]v) {(b)};
\end{tikzpicture}
\begin{tikzpicture}[scale=1,
label/.style={rectangle,draw,black,fill=white,text=black,inner sep=1pt}]
\node (v) at (1.75,1.75) {\tiny $\;$};

\node[circle,draw=red,text=red,text=red,thick,fill=red!10,inner sep=0pt,minimum size=20pt] at ([shift={(-.75,0)}]v) {$x$};
\node[circle,draw=blue,text=blue,text=blue,thick,fill=blue!10,inner sep=0pt,minimum size=20pt] at ([shift={(-.75,-1)}]v) {$y$};

\path[->,>=stealth,very thick,color=blue] (1.25,1)  edge node {} (2,1.75);
\node at (1.7,1.52) [label] (l1) {\small 5 };
\path[->,>=stealth,very thick,color=red] (1.25,1.5) edge node {} (2,.75);
\node at (1.71,1.03) [label] (l1) {\small 2 };
\path[->,>=stealth,very thick,color=blue] (1.1,1.1)  edge node {} (1.1,1.75);
\node at (1.2,1.27) [label] (l1) {\small 4 };

\node at (1.8,1.85) [rectangle,inner sep=0pt] {\scriptsize[0, 0.71)};
\node at (.65,1.25) [rectangle,inner sep=0pt] {\scriptsize[0, 0.58)};

\node at ([shift={(-.45,-1.8)}]v) {(c)};
\end{tikzpicture}
\begin{tikzpicture}[scale=1.6,
label/.style={rectangle,draw,black,fill=white,text=black,inner sep=1pt}]

\node at (1,2.5) [label] (l1) {\small 1 };
\node at (1,2.18) [rectangle,inner sep=0pt] {\scriptsize[0, 0.58)};
\node at (1,2) [label,very thick] (l2) {\small 2 };
\node at (1,1.5) [label,very thick] (l3) {\small 3 };

\node at (2,2.5) [label,very thick] (l4) {\small 4 };
\node at (2,2) [label,very thick] (l5) {\small 5 };
\node at (2,1.5) [label] (l6) {\small 6 };

\draw[-,draw=gray] (l1) -- (l5);
\draw[-,draw=gray ] (l2) -- (l6);
\draw[-,very thick ] (l2) -- (l4);
\draw[-,very thick ] (l2) -- (l5);
\draw[-,very thick ] (l3) -- (l4);
\draw[-,very thick ] (l3) -- (l5);

\node at (1.5,1.2) [rectangle,fill=white,inner sep=0pt] {(d)};
\node at (1,1.05) [rectangle,text=white,inner sep=0pt] {.};
\end{tikzpicture}
\caption{An example of bipartite conflict analysis. (a) A MAPF instance, (b) an enumeration of actions (wait actions omitted), (c) an illustration of unsafe intervals for action 2 and (d) a bipartite conflict graph and biclique.}
\label{fig:BC}
\end{figure}
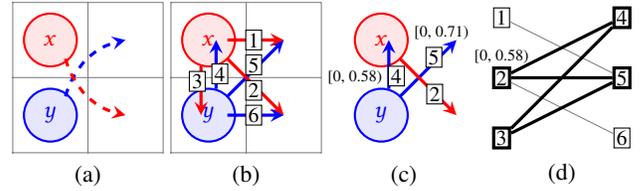


Figure \ref{fig:BC}(a) shows an example problem where two agents must cross paths. Figure \ref{fig:BC}(b) shows an enumeration of all actions available to two agents at overlapping timeframes, (wait actions omitted). Figure \ref{fig:BC}(d) shows the bipartite conflict graph for the enumerated actions. A bipartite conflict graph (BCG)~\cite{walker2020generalized} is constructed by creating two sets of nodes for the actions available to the two agents during overlapping timeframes. The nodes for one agent are arranged on the left, and nodes for the other agent are arranged on the right. Then edges are added between nodes for pairs of actions that conflict. The graph is bipartite because no node on the left is connected to any other node on the left, similarly for nodes on the right, but nodes on the left may be connected to nodes on the right.

In order to choose a mutually disjunctive set of actions to use as constraints, the nodes chosen must form a bipartite clique or \emph{biclique} in the BCG. That is, every node chosen for the set on the left must be connected to every node chosen for the set on the right of the BCG. A biclique is shown with thick lines in Figure \ref{fig:BC}(d), where the set of nodes in the biclique is $\{2,3\}$ and $\{4,5\}$ respectively. In practice, one can find a \emph{max-vertex biclique} (a biclique with a maximal number of vertices) in polynomial time~\cite{walker2020generalized}.

\label{sec:unsafe-interval}Biclique nodes in a BCG can be annotated with unsafe intervals for SIPP. This is done by computing the unsafe times (i.e., the amount of time one agent should wait to avoid conflict) for each neighboring node in the biclique as shown for action 2 in Figure \ref{fig:BC}(c). Unsafe interval computation can be done using a binary search approach~\cite{AASIPP} or, for circular agents, using an algebraic approach~\cite{walker2019collision}. In this example, action 2 cannot be performed in the time range $[0,0.58)$ in order to avoid conflict with action 4, and $[0,0.71)$ to avoid conflict with action 5. The intersection of those intervals is used to annotate the action in the biclique as shown for action 2 in Figure \ref{fig:BC}(d). Avoiding execution of the action in the intersected unsafe interval ensures the mutually-conflicting property~\cite{walker2020generalized}.

In the example, blocking action 2 in the intersected interval ($[0,0.58)\cap[0,0.71){=}[0,0.58)$) ensures that the interval for which action 2 is blocked conflicts with all other actions in the biclique, namely actions 4 and 5. If the interval $[0,0.71)$ were erroneously chosen instead, action 2 would be blocked for part of the timeframe ($[.58,.71)$) in which it does not conflict with action 4, potentially blocking a feasible and/or optimal solution that uses action 4 in that timeframe. This technique is necessary for creating valid sets of time-range constraints required by the SIPP routine of CCBS.

\section{New Enhancements}

We now discuss new enhancements and their implementation.

\subsection{Bypass}

The bypass enhancement (BP)~\cite{CBSBP} is a symmetry breaking technique used to avoid some splits in the CT. BP was never implemented for CCBS, and to our knowledge, no study has been performed to determine its effectiveness. Our intent in including BP in this paper is to experiment with its effectiveness in continuous-time domains. Our findings show that it is very effective in problem instances with similarities to ``classic'' MAPF, but much less effective in certain continuous-time settings. These details will be discussed in the empirical results section.

The implementation of BP for CCBS is straightforward, and follows the original formulation, with some adaptations for continuous-time. BP (for both CBS and CCBS) inspects the paths for two agents involved in a conflict. If a new path of the same cost is available for one of the conflicting agents such that: (1) the new path does not have an increased cost, (2) the new path respects new constraints which are required to avoid the conflict that caused the split and (3) the new path has fewer conflicts with all agents than the respective path in the parent node, this new path is called a \emph{bypass}. If a bypass is found, child nodes are not generated, instead, the current node is updated with the bypass path for one of the agents and re-inserted into the OPEN list. This enhancement improves performance by avoiding splits in the tree which would otherwise result in two new sub-trees.

\subsection{New Biclique Constraints for CCBS}
In this paper, we tested BC with CCBS for the first time. Although the use of max-vertex bicliques for BC was shown to be very effective in continuous-time domains with \textit{fixed} wait actions~\cite{walker2020generalized}, when applied to CCBS, which computes \textit{arbitrary} wait actions, we found that using max-vertex bicliques to generate biclique constraints was usually detrimental to performance. As noted earlier, taking the intersection of unsafe intervals of adjacent edges usually causes the interval to be shortened. Because of this shortened interval, the resulting safe intervals used with SIPP will not cause wait actions to be generated with a long enough duration to avoid conflict. This can result in causing the conflict between two actions which caused the split (which we will call the \emph{core action pair}) to recur at a slightly later time, resulting in another split in the sub-tree for the same two actions.

In order to remedy this, we computed the biclique so that it only includes the pairs of actions whose unsafe interval is a superset of that for the \emph{core action pair}. For example, if the core action pair were actions 2 and 4, the biclique would include both actions 4 and 5, because its unsafe interval is a superset: $[0,0.71)\supseteq[0,0.58)$. On the other hand, if the core action pair were actions 2 and 5, action 4 could not be included in the set. In this way, the correct wait time is generated by SIPP and the same conflict is always avoided in the sub-tree. This approach of computing an \emph{interval-superset biclique} often results in a smaller biclique and thus a smaller set of biclique constraints than the max-vertex biclique approach, nevertheless, the result yields performance gains versus CCBS with the original splitting method in many settings.

Biclique constraints in CBS are not usually applicable in ``classic'' planar graphs (assuming agents' size is sufficiently small) since agents would only conflict with one other action at a time (e.g., agents crossing the same edge in opposite directions), resulting in a 1x1 biclique which is equivalent to classic constraints. However, with CCBS, biclique constraints are useful in planar graphs since multiple wait actions are possible for the same agent at a vertex at a specific point in time, which makes multiple conflicts possible.

\subsection{Disjoint Splitting with Biclique Constraints}

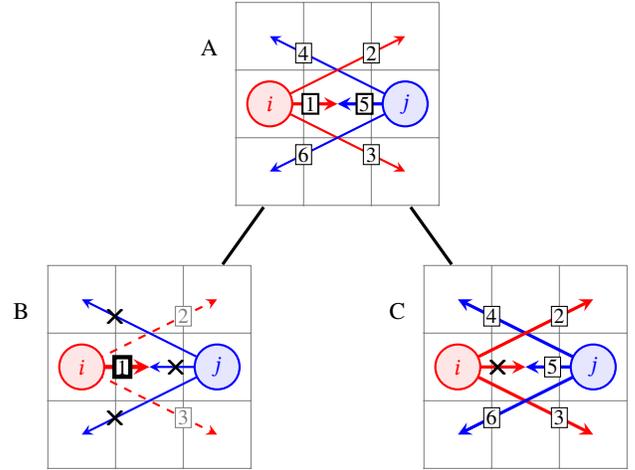
\begin{figure}[t!]
\begin{tikzpicture}

\node[inner sep=0pt] (A) at (0,0) {
\begin{tikzpicture}[scale=.9,
label/.style={rectangle,draw,black,fill=white,text=black,inner sep=1pt}]
\node (v) at (1.75,1.75) {\tiny $\;$};
\draw[step=1cm,color=gray,shift={(.5,0.25)}] ([shift={(-1.26,-1.51)}]v) grid (3,3);

\path[->,>=stealth,very thick,color=blue] (2.75,1.75)  edge node {} (2,1.75);
\node at (2.4,1.75) [label,thick] (l1) {\small 5 };
\path[->,>=stealth,thick,color=red] (1,1.75) edge node {} (3,.75);
\node at (2.5,1.0) [label] (l1) {\small 3 };
\path[->,>=stealth,thick,color=blue] (3,1.75)  edge node {} (1,.75);
\node at (1.5,1.0) [label] (l1) {\small 6 };
\path[->,>=stealth,very thick,color=red] (1.25,1.75) edge node {} (2,1.75);
\node at (1.6,1.75) [label,thick] (l1) {\small 1 };
\path[->,>=stealth,thick,color=blue] (3,1.75)  edge node {} (1,2.75);
\node at (1.5,2.5) [label] (l1) {\small 4 };
\path[->,>=stealth,thick,color=red] (1,1.75) edge node {} (3,2.75);
\node at (2.5,2.5) [label] (l1) {\small 2 };

\node[circle,draw=red,text=red,text=red,thick,fill=red!10,inner sep=0pt,minimum size=17pt] at ([shift={(-.75,0)}]v) {$i$};
\node[circle,draw=blue,text=blue,text=blue,thick,fill=blue!10,inner sep=0pt,minimum size=17pt] at ([shift={(1.25,0)}]v) {$j$};

\end{tikzpicture}
};

\node[inner sep=0pt] (B) at (-2.5,-3.5) {
\begin{tikzpicture}[scale=.9,
label/.style={rectangle,draw,black,fill=white,text=black,inner sep=1pt}]
\node (v) at (1.75,1.75) {\tiny $\;$};
\draw[step=1cm,color=gray,shift={(.5,0.25)}] ([shift={(-1.26,-1.51)}]v) grid (3,3);

\path[->,>=stealth,thick,color=blue] (2.75,1.75)  edge node {} (2,1.75);
\node at (2.4,1.75) [inner sep=0pt] (l1) {\Large $\pmb{\times}$ };
\path[->,>=stealth,thick,dashed,color=red] (1,1.75) edge node {} (3,.75);
\node at (2.5,1.0) [label,draw=gray,text=gray] (l1) {\small 3 };
\path[->,>=stealth,thick,color=blue] (3,1.75)  edge node {} (1,.75);
\node at (1.5,1.0) [inner sep=0pt] (l1) {\Large $\pmb{\times}$};
\path[->,>=stealth,ultra thick,color=red] (1.25,1.75) edge node {} (2,1.75);
\node at (1.6,1.75) [label,ultra thick] (l1) {\small 1 };
\path[->,>=stealth,thick,color=blue] (3,1.75)  edge node {} (1,2.75);
\node at (1.5,2.5) [inner sep=0pt] (l1) {\Large $\pmb{\times}$};
\path[->,>=stealth,thick,dashed,color=red] (1,1.75) edge node {} (3,2.75);
\node at (2.5,2.5) [label,draw=gray,text=gray] (l1) {\small 2 };

\node[circle,draw=red,text=red,text=red,thick,fill=red!10,inner sep=0pt,minimum size=17pt] at ([shift={(-.75,0)}]v) {$i$};
\node[circle,draw=blue,text=blue,text=blue,thick,fill=blue!10,inner sep=0pt,minimum size=17pt] at ([shift={(1.25,0)}]v) {$j$};
\end{tikzpicture}
};

\node[inner sep=0pt] (C) at (2.5,-3.5) {
\begin{tikzpicture}[scale=.9,
label/.style={rectangle,draw,black,fill=white,text=black,inner sep=1pt}]
\node (v) at (1.75,1.75) {\tiny $\;$};
\draw[step=1cm,color=gray,shift={(.5,0.25)}] ([shift={(-1.26,-1.51)}]v) grid (3,3);

\path[->,>=stealth,very thick,color=blue] (2.75,1.75)  edge node {} (2,1.75);
\node at (2.4,1.75) [label] (l1) {\small 5 };
\path[->,>=stealth,very thick,color=red] (1,1.75) edge node {} (3,.75);
\node at (2.5,1.0) [label] (l1) {\small 3 };
\path[->,>=stealth,very thick,color=blue] (3,1.75)  edge node {} (1,.75);
\node at (1.5,1.0) [label] (l1) {\small 6 };
\path[->,>=stealth,very thick,color=red] (1.25,1.75) edge node {} (2,1.75);
\node at (1.6,1.75) [inner sep=0pt] (l1) {\Large $\pmb{\times}$};
\path[->,>=stealth,very thick,color=blue] (3,1.75)  edge node {} (1,2.75);
\node at (1.5,2.5) [label] (l1) {\small 4 };
\path[->,>=stealth,very thick,color=red] (1,1.75) edge node {} (3,2.75);
\node at (2.5,2.5) [label] (l1) {\small 2 };

\node[circle,draw=red,text=red,text=red,thick,fill=red!10,inner sep=0pt,minimum size=17pt] at ([shift={(-.75,0)}]v) {$i$};
\node[circle,draw=blue,text=blue,text=blue,thick,fill=blue!10,inner sep=0pt,minimum size=17pt] at ([shift={(1.25,0)}]v) {$j$};
\end{tikzpicture}
};

\path[-, very thick] (A) edge (B);
\path[-, very thick] (A) edge (C);
\node at (-1.7,.75) {A};
\node at (-4.2,-2.75) {B};
\node at (.8,-2.75) {C};

\end{tikzpicture}
\caption{An example of a disjoint split with biclique constraints.}
\label{fig:DJBC}
\end{figure}

Recall that with disjoint splitting, in one CT child node there is a positive constraint for agent $i$, and in the other there is a negative constraint for agent $i$. Additionally, (specifically in the formulation for CCBS), a single negative constraint for agent $j$ is added to the node with the positive constraint. Doing this alone is quite effective~\cite{andreychuk2021improving}. However, it is still possible for a conflict with the positively constrained action to recur in the sub-tree of the CT. We propose that when performing a disjoint split between agent $i$ and agent $j$ that along with a positive constraint for agent $i$, a \emph{set} of negative constraints be added for agent $j$ to the same CT node so that while agent $i$ is forced to take an action, agent $j$ is forced to avoid \textit{all actions} which conflict with it.

To comprehensively enforce that agent $j$ avoids the positively-constrained action for agent $i$ (or vice-versa), we perform bipartite analysis. The procedure is outlined below:

\begin{enumerate}
    \item Construct a BCG for a core-action pair $\langle a_i, a_j\rangle$ as follows:
    \begin{enumerate}
        \item Create vertices for all of agent $j$'s actions that conflict with $a_i$ on the left.
        \item Create vertices for all of agent $i$'s actions that conflict with $a_j$ on the right.
        \item Add an edge from vertex $a_i$ to all vertices on the right.
        \item Add an edge from vertex $a_j$ to all vertices on the left.
    \end{enumerate}
    \item For each edge, compute the unsafe interval and annotate the edge with it.
    \item Select either agent $i$ or agent $j$ to receive a positive constraint.
    \item Select all vertices connected to the constrained agent's core action vertex to create negative constraints.
    \item Perform interval intersection for the positive constraint only, based on edges adjacent to the constrained agent's core action vertex.
\end{enumerate}

This procedure for computing time-annotated biclique constraints is simpler than outlined in the previous section because it is not necessary to test for the addition of edges for all pairs of vertices in the BCG and therefore the interval intersection and shortening steps are not necessary for the negative constraints. This is because the negative constraints must avoid the positive constraint's action for the entire duration. In step (3), our implementation chooses the negatively constrained agent to be the one with the most nodes. However, alternative approaches could be taken such as choosing the set with the largest cumulative unsafe interval, the largest mean unsafe interval, etc. Step (5) is necessary to obtain the proper unsafe interval for SIPP.

The placement of these constraints is illustrated in Figure \ref{fig:DJBC}. Node A is a node in the CT with a conflict between actions 1 and 5 shown in bold; this is the core action pair. Nodes B and C are child nodes. Node B shows the positive constraint for the red agent in bold for action 1, with negative constraints for the blue agent's conflicting actions shown with `x's. The other actions for the red agent in node B are dashed, meaning that they are no longer reachable because of the positive constraint. Finally, node C shows a single negative constraint that mirrors the positive constraint in node B.

In summary, while regular disjoint splitting would only create a positive constraint for agent $i$ with a single negative constraint for agent $j$ to child node 1 and a negative constraint for agent $i$ to child node 2, disjoint splitting with bicliques, or \emph{disjoint bicliques} (DB) adds multiple negative constraints for agent $j$ to child node 1 as well. This approach effectively eliminates further conflicts between agent $i$ and agent $j$ at the positively constrained action in the CT sub-tree. The effect of adding the extra constraints for agent $j$ is that agent $j$ avoids the positively-constrained action for agent $i$ from multiple paths, preemptively avoiding conflicts further down in the CT. Ultimately, a potentially exponential number of child nodes are pruned from the CT.

We now show that this approach is complete (if a solution to the problem instance exists) and that it also ensures optimality.

\begin{lemma}
\label{lem:complete}
Using biclique constraints with disjoint splitting (disjoint bicliques) never blocks a feasible solution from being found by CCBS.
\end{lemma}
\begin{proof}
  Let $N$ be a CT node. Let $a_i$ and $a_j$ be a core action pair -- a conflicting pair of actions from $\pi_i\in N.\Pi$ and $\pi_j\in N.\Pi$ respectively.

  Let $\bar{N_i}$ be a child of $N$ which contains a single negative constraint blocking action $a_i$. Let $\hat{N_i}$ be a second child of $N$ with a single positive constraint forcing agent $i$ to perform action $a_i$, and multiple negative constraints $\bar{C_j}$, blocking agent $j$ from conflicting with $a_i$.
    
  Let $\Pi^*$ be the only feasible solution to the MAPF problem instance. There are three possible cases:
  \begin{enumerate}
      \item $\Pi^*$ contains $a_i$
      \item $\Pi^*$ contains $a_j$
      \item $\Pi^*$ contains neither $a_i$ nor $a_j$
  \end{enumerate}
  If case 1 is true, $\Pi^*$ is guaranteed to be found in the sub-tree of $\hat{N_i}$ because $a_i$ is enforced by the positive constraint. If case 2 is true, $\Pi^*$ is guaranteed to be found in the sub-tree of $\bar{N_i}$ because $a_j$ is not blocked. If case 3 is true, $\Pi^*$ is guaranteed to be found in the sub-tree of $\bar{N_i}$ because $a_i$ is blocked and $a_j$ is not enforced.

  Because the BCG only includes actions which conflict with $a_i$, it is not possible to block any action that does not conflict with $a_i$. Since in case 1, $\Pi^*$ cannot contain any actions that conflict with $a_i$, blocking all actions for agent $j$ which conflict with $a_i$ cannot preclude $\Pi^*$. Thus case 1 still holds with disjoint bicliques. 
\end{proof}

\begin{theorem}
  \label{thm:optimal}
  CCBS with disjoint bicliques ensures optimality.
\end{theorem}
\begin{proof}
    Per Lemma \ref{lem:complete}, using disjoint bicliques can never preclude CCBS from finding a solution (if any exist). Assuming the OPEN list of CCBS is ordered by lowest cost, CCBS is guaranteed to find a lowest-cost feasible solution first before terminating.
\end{proof}

\subsection{Disjoint K-Partite Cliques}

So far, we have discussed the approach for combining bicliques with disjoint splitting for resolving a single conflict. It is often the case that multiple agents conflict with the positively constrained action. It is possible (and helpful) to additionally constrain these agents using negative constraints. This is done by performing step (1) outlined for DB for all agents that have a conflict with the positively constrained action to form two $k$-partite conflict graphs (KCG), one for agent $i$ and another for agent $j$. But it is explained more simply as enumerating all actions by all agents which conflict with $a_i$ and $a_j$ respectively and computing unsafe intervals. All other steps are straightforward. In step (3), our implementation chooses the agent with the largest KCG to get the positive constraint.

In order to avoid performing extra conflict checks to discover other agents that conflict with the positively-constrained action, we make use of a conflict count table (CCT)~\cite{walker2022multi} which is adapted for MAPF\textsubscript{R} to perform bookkeeping on all conflicts. The CCT is set up so that looking up conflicts in the table is indexed on a per-agent, per-action basis to expedite the KCG creation. In a vast majority of cases, the number of partitions in the KCG graphs are much smaller than the number of agents.

We show that DK is correct by a simple extension of Lemma \ref{lem:complete}:

\begin{lemma}
    \label{lem:completeK}
    The use of disjoint k-partite cliques with disjoint splitting never blocks a feasible solution in CCBS.
\end{lemma}
\begin{proof}
    Following the definitions from Lemma \ref{lem:complete}, DK now adds negative constraints for multiple agents to $\hat{N_i}$. Since the constraints in $\bar{N_i}$ are unchanged, cases 2 and 3 still hold.

    Case 1 still holds because $\Pi^*$ cannot contain any action by any agent which conflicts with $a_i$, therefore just as the disjoint bicliques procedure ensures that only actions which conflict with $a_i$ are blocked for agent $j$, DK ensures that only actions which conflict with $a_i$ are blocked for all agents $j\neq i$ (or subset of agents $\in j\neq i$). Thus, DK can never block a feasible solution in any of the three cases.
\end{proof}

Finally, substituting Lemma \ref{lem:completeK} into Theorem \ref{thm:optimal}, we see that DK is also optimal. In summary, we now have a powerful capability for multi-agent, symmetry-breaking constraints, capable of eliminating even more nodes in the CT.

\section{Empirical Results}
\label{sec:results}

\begin{figure}[t!]
\centering
\begin{tikzpicture}
\begin{groupplot}[
    group style={
        group size=1 by 9,
        ylabels at=edge left,
        vertical sep=20pt,
        horizontal sep=0pt
    },
    xtick pos=left,
    ytick pos=left,
    ylabel near ticks,
    xlabel near ticks,
    footnotesize,
    width=7.0cm,
    height=3.3cm,
    ytick align=outside,
    xtick align=outside,
    grid=both,
    grid style={line width=.1pt},
    legend columns=6,
    legend style={at={(-.1,1.05)},draw=none,fill=none,column sep=.1ex,anchor=south west,font=\tiny},
    title style={at={(.6,1.6)},font=\large},
    tick label style={font=\tiny}
]

\nextgroupplot[ymax=100,ytick={0,50,100},xmin=5,xmax=50,xtick={0,10,...,100},ylabel=Success Rate (\%)]
\addplot[mark=square*,color=pink] coordinates {(4, 100) (8, 92) (12, 84) (16, 68) (20, 40) (24, 32) (28, 16) (32, 12)};
\addplot[mark=square*,color=brown] coordinates {(4, 100) (8, 100) (12, 100) (16, 96) (20, 92) (24, 72) (28, 52) (32, 28) (36, 12) (40, 8)};
\addplot[mark=*,color=teal] coordinates {(4, 100) (8, 92) (12, 84) (16, 72) (20, 44) (24, 32) (28, 16) (32, 4)};
\addplot[mark=triangle*,color=cyan] coordinates {(4, 100) (8, 100) (12, 100) (16, 96) (20, 92) (24, 72) (28, 52) (32, 28) (36, 12) (40, 8)};
\addplot[mark=triangle*,color=red] coordinates {(4, 100) (8, 100) (12, 92) (16, 80) (20, 52) (24, 40) (28, 20) (32, 12) (36, 4)};
\addplot[mark=diamond*,color=black] coordinates {(4, 100) (8, 100) (12, 100) (16, 96) (20, 88) (24, 76) (28, 64) (32, 40) (36, 24) (40, 12) (44, 8) (48, 4)};
\legend {CCBS,Base,BP,BP+DS,BP+BC,BP+DK};

\nextgroupplot[ymax=100,ytick={0,50,100},xmin=2,xmax=80,xtick={0,10,...,100},ylabel=Success Rate (\%)]
\addplot[mark=square*,color=pink] coordinates {(10, 100) (20, 96) (30, 76) (40, 36) (50, 16) (60, 8) (70, 4)};
\addplot[mark=square*,color=brown] coordinates {(10, 100) (20, 100) (30, 100) (40, 88) (50, 48) (60, 32) (70, 12)};
\addplot[mark=triangle*,color=cyan] coordinates {(10, 100) (20, 100) (30, 100) (40, 88) (50, 48) (60, 32) (70, 12)};
\addplot[mark=*,color=teal] coordinates {(10, 100) (20, 96) (30, 84) (40, 44) (50, 20) (60, 8) (70, 4)};
\addplot[mark=triangle*,color=red] coordinates {(10, 100) (20, 96) (30, 84) (40, 48) (50, 24) (60, 8) (70, 4)};
\addplot[mark=diamond*,color=black] coordinates {(10, 100) (20, 100) (30, 100) (40, 88) (50, 64) (60, 40) (70, 20) (80, 4)};

\nextgroupplot[ymax=100,ytick={0,50,100},xmin=2,xmax=80,xtick={0,10,...,100},ylabel=Success Rate (\%)]
\addplot[mark=square*,color=pink] coordinates {(10, 100) (20, 96) (30, 96) (40, 52) (50, 24) (60, 8)};
\addplot[mark=square*,color=brown] coordinates {(10, 100) (20, 96) (30, 96) (40, 80) (50, 48) (60, 20) (70, 4)};
\addplot[mark=*,color=teal] coordinates {(10, 100) (20, 96) (30, 96) (40, 56) (50, 20) (60, 4)};
\addplot[mark=triangle*,color=cyan] coordinates {(10, 100) (20, 96) (30, 96) (40, 80) (50, 48) (60, 20) (70, 4)};
\addplot[mark=triangle*,color=red] coordinates {(10, 100) (20, 96) (30, 96) (40, 72) (50, 24) (60, 8)};
\addplot[mark=diamond*,color=black] coordinates {(10, 100) (20, 100) (30, 96) (40, 96) (50, 64) (60, 32) (70, 12) (80, 0)};

\nextgroupplot[ymax=100,ytick={0,50,100},xmin=2,xmax=30,xtick={0,10,...,100},ylabel=Success Rate (\%)]
\addplot[mark=square*,color=pink] coordinates {(4, 96) (8, 84) (12, 40) (16, 12) (20, 12) (24, 4)};
\addplot[mark=square*,color=brown] coordinates {(4, 100) (8, 96) (12, 56) (16, 28) (20, 16) (24, 4) (28, 4)};
\addplot[mark=*,color=teal] coordinates {(4, 96) (8, 84) (12, 40) (16, 16) (20, 12) (24, 4)};
\addplot[mark=triangle*,color=cyan] coordinates {(4, 100) (8, 96) (12, 56) (16, 28) (20, 16) (24, 4) (28, 4)};
\addplot[mark=triangle*,color=red] coordinates {(4, 100) (8, 88) (12, 40) (16, 20) (20, 16) (24, 4)};
\addplot[mark=diamond*,color=black] coordinates {(4, 100) (8, 100) (12, 84) (16, 36) (20, 16) (24, 4) (28, 4)};

\nextgroupplot[ymax=100,ytick={0,50,100},xmin=2,xmax=80,xtick={0,10,...,100},ylabel=Success Rate (\%)]
\addplot[mark=square*,color=pink] coordinates {(10, 100) (20, 96) (30, 84) (40, 44) (50, 4)};
\addplot[mark=square*,color=brown] coordinates {(10, 100) (20, 100) (30, 96) (40, 76) (50, 28) (60, 12)};
\addplot[mark=*,color=teal] coordinates {(10, 100) (20, 96) (30, 84) (40, 44) (50, 4)};
\addplot[mark=triangle*,color=cyan] coordinates {(10, 100) (20, 100) (30, 96) (40, 76) (50, 28) (60, 12)};
\addplot[mark=triangle*,color=red] coordinates {(10, 100) (20, 96) (30, 88) (40, 52) (50, 8)};
\addplot[mark=diamond*,color=black] coordinates {(10, 100) (20, 100) (30, 100) (40, 92) (50, 48) (60, 12) (70, 0)};

\nextgroupplot[ymax=100,ytick={0,50,100},xmin=5,xmax=90,xtick={0,10,...,100},ylabel=Success Rate (\%)]
\addplot[mark=square*,color=pink] coordinates {(10, 100) (20, 100) (30, 92) (40, 72) (50, 36) (60, 16) (70, 12) (80, 4)};
\addplot[mark=square*,color=brown] coordinates {(10, 100) (20, 100) (30, 96) (40, 88) (50, 56) (60, 36) (70, 24) (80, 4)};
\addplot[mark=*,color=teal] coordinates {(10, 100) (20, 100) (30, 92) (40, 80) (50, 44) (60, 20) (70, 12) (80, 4)};
\addplot[mark=triangle*,color=cyan] coordinates {(10, 100) (20, 100) (30, 96) (40, 88) (50, 56) (60, 36) (70, 24) (80, 4)};
\addplot[mark=triangle*,color=red] coordinates {(10, 100) (20, 100) (30, 92) (40, 84) (50, 48) (60, 20) (70, 12) (80, 4)};
\addplot[mark=diamond*,color=black] coordinates {(10, 100) (20, 100) (30, 100) (40, 92) (50, 76) (60, 52) (70, 32) (80, 20) (90, 0)};

\nextgroupplot[ymax=100,ytick={0,50,100},xmin=2,xmax=30,xtick={0,10,...,100},xlabel=Number of Agents,ylabel=Success Rate (\%)]
\addplot[mark=square*,color=pink] coordinates {(4, 100) (8, 72) (12, 52) (16, 24) (20, 4)};
\addplot[mark=square*,color=brown] coordinates {(4, 100) (8, 92) (12, 84) (16, 52) (20, 36) (24, 12)};
\addplot[mark=*,color=teal] coordinates {(4, 100) (8, 72) (12, 52) (16, 24) (20, 4)};
\addplot[mark=triangle*,color=cyan] coordinates {(4, 100) (8, 92) (12, 84) (16, 52) (20, 36) (24, 12)};
\addplot[mark=triangle*,color=red] coordinates {(4, 100) (8, 88) (12, 68) (16, 36) (20, 16)};
\addplot[mark=diamond*,color=black] coordinates {(4, 100) (8, 96) (12, 92) (16, 76) (20, 48) (24, 32) (28, 4)};

\end{groupplot}


\node[inner sep=0pt] at (4.7,1.5) {\small den520d};
\node[inner sep=0pt] at (6.4,.86)
    {\includegraphics[width=1.72cm]{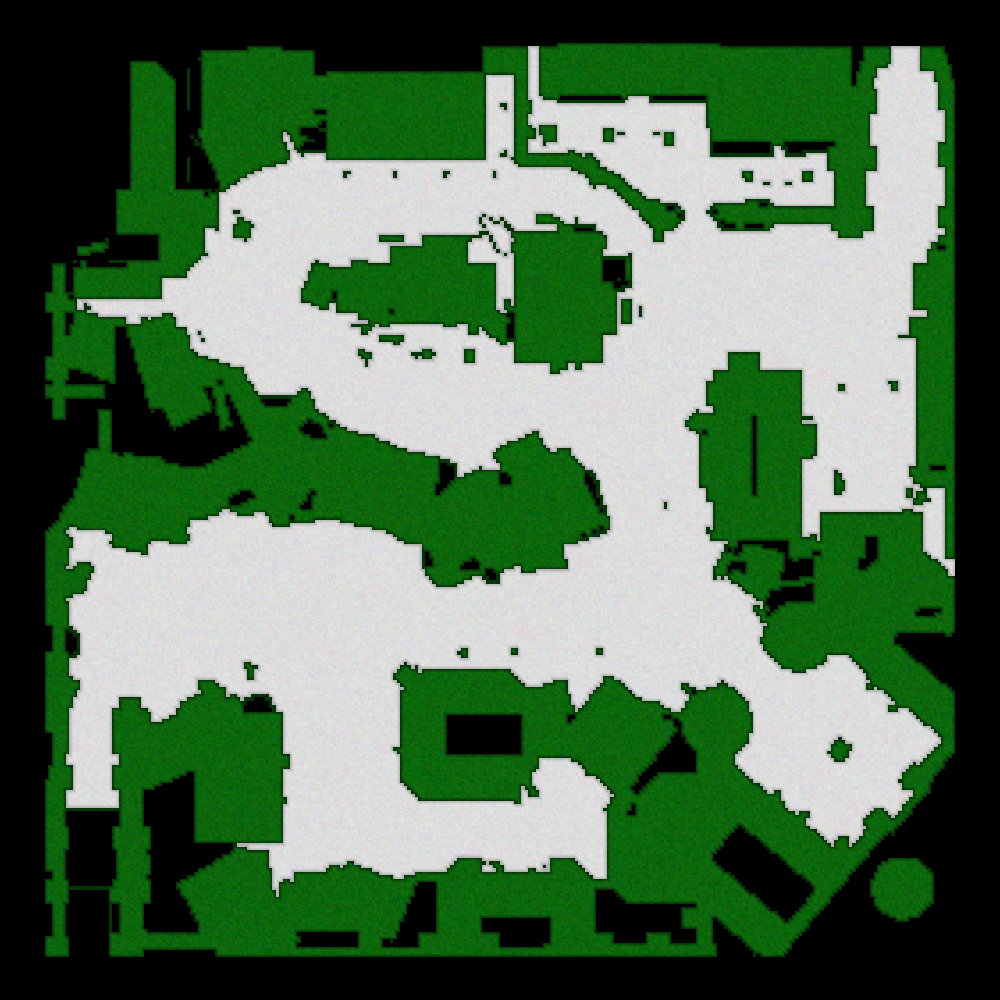}};

\node[inner sep=0pt] at (4.4,-.93) {\small Boston\_0\_256};
\node[inner sep=0pt] at (6.4,-1.55)
    {\includegraphics[width=1.72cm]{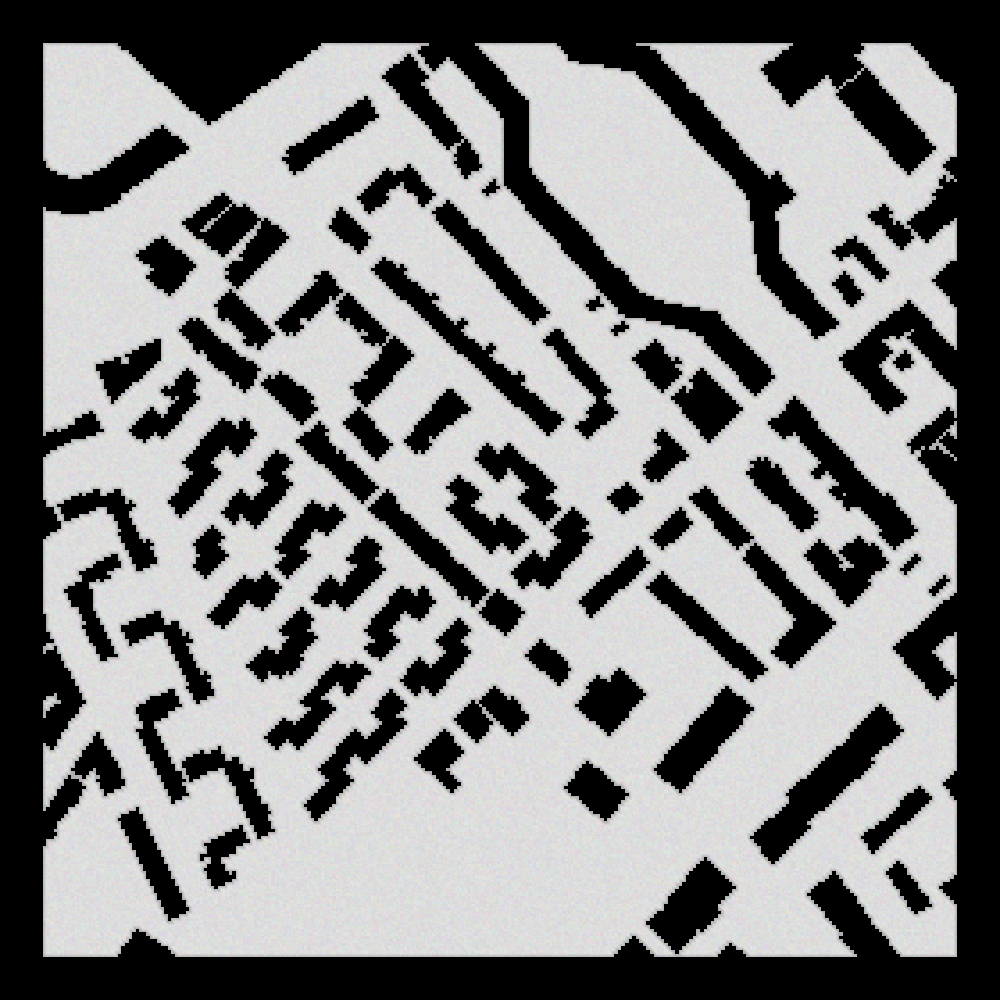}};
    
\node[inner sep=0pt] at (4.5,-3.37) {\small empty-48-48};
\node[inner sep=0pt] at (6.4,-3.95)
    {\includegraphics[width=1.72cm]{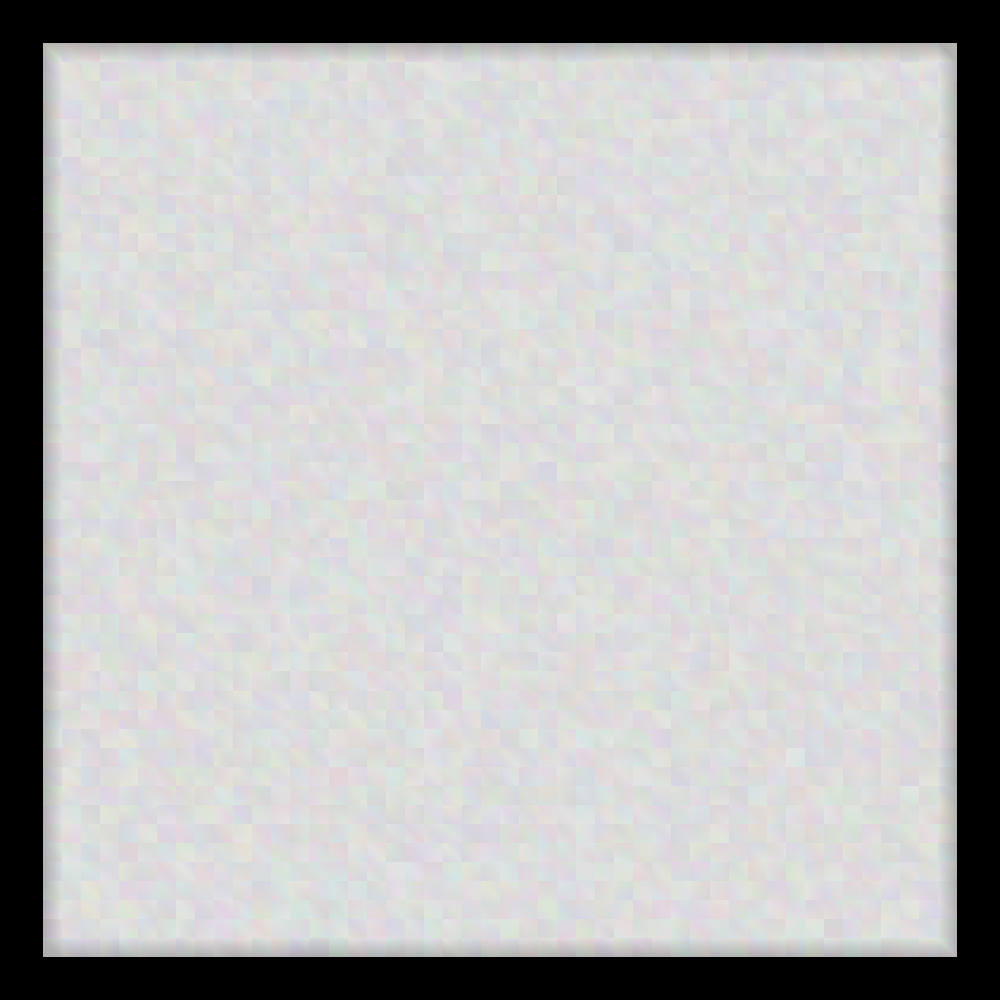}};

\node[inner sep=0pt] at (4.2,-5.75) {\small maze-128-128-10};
\node[inner sep=0pt] at (6.4,-6.40)
    {\includegraphics[width=1.72cm]{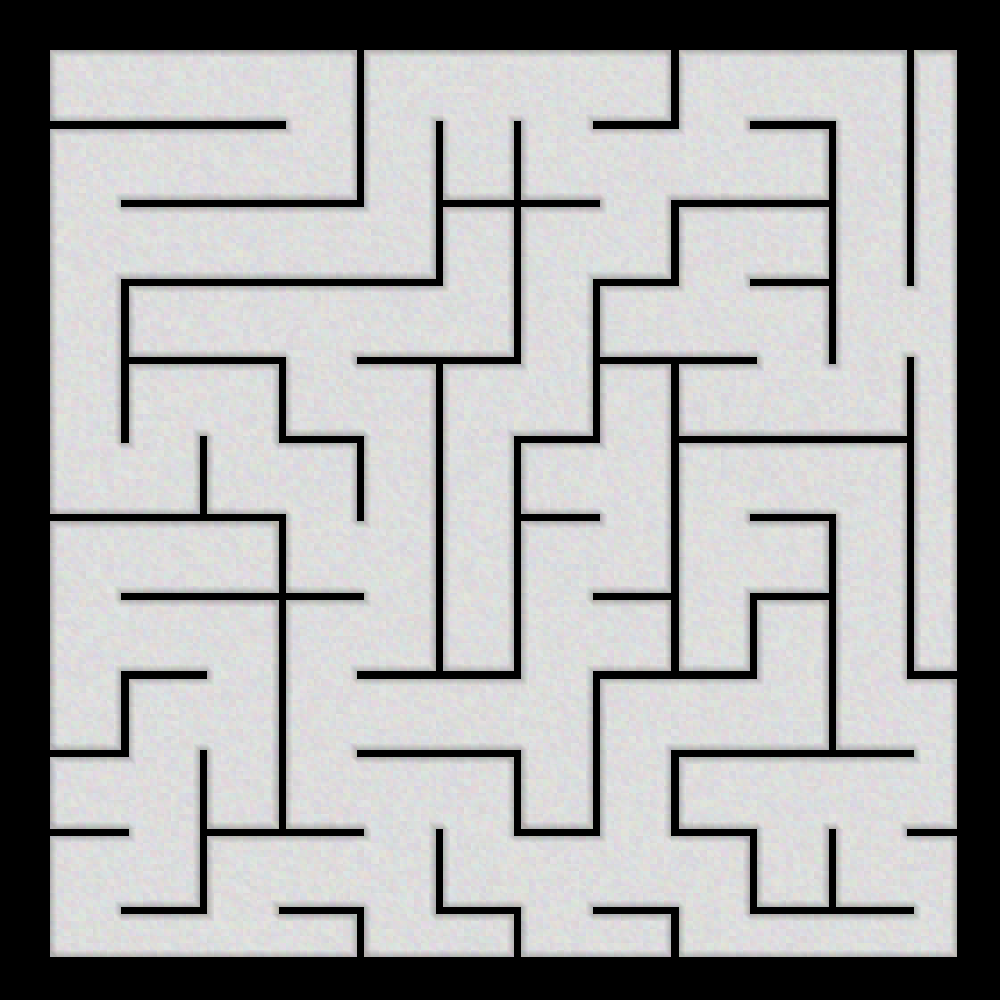}};

\node[inner sep=0pt] at (4.2,-8.15) {\small random-64-64-10};
\node[inner sep=0pt] at (6.4,-8.80)
    {\includegraphics[width=1.72cm]{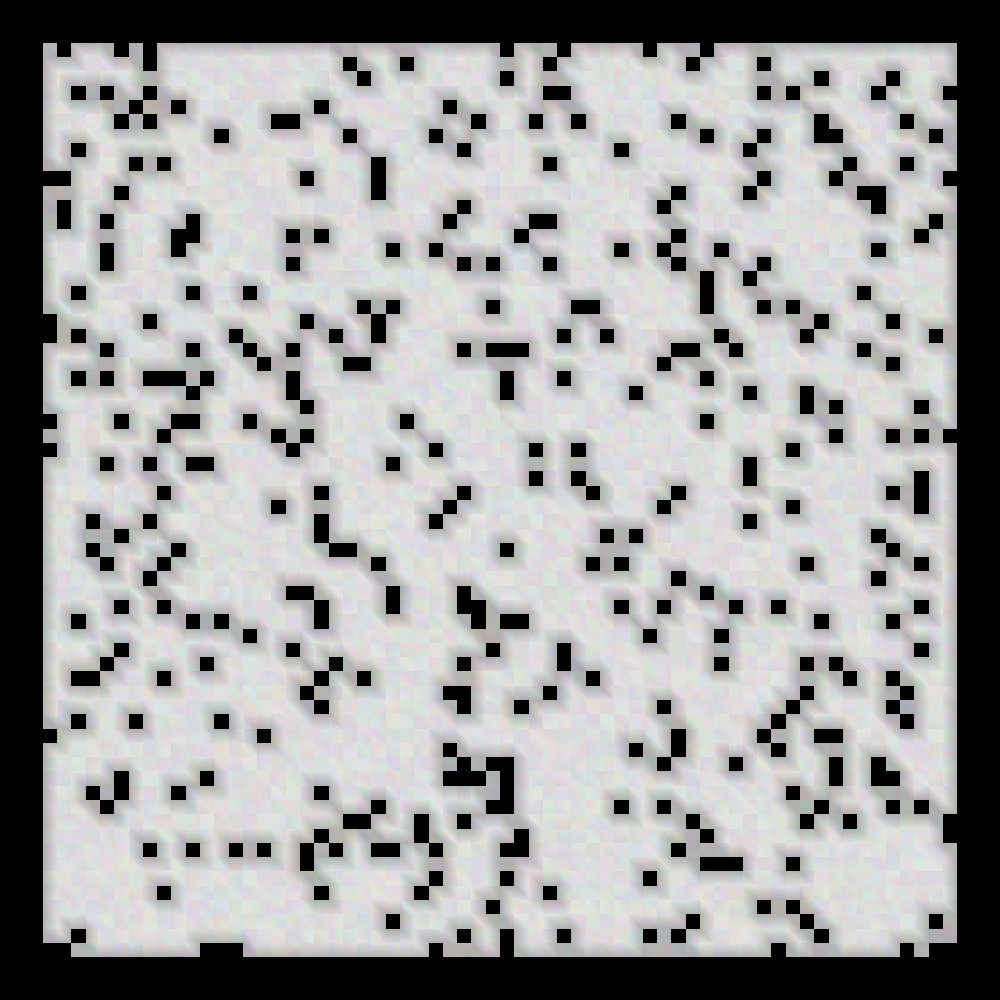}};
    
\node[inner sep=0pt] at (4.15,-10.57) {\scriptsize warehouse-10-20-10-2-2};
\node[inner sep=0pt] at (6.4,-11.25)
    {\includegraphics[width=1.72cm]{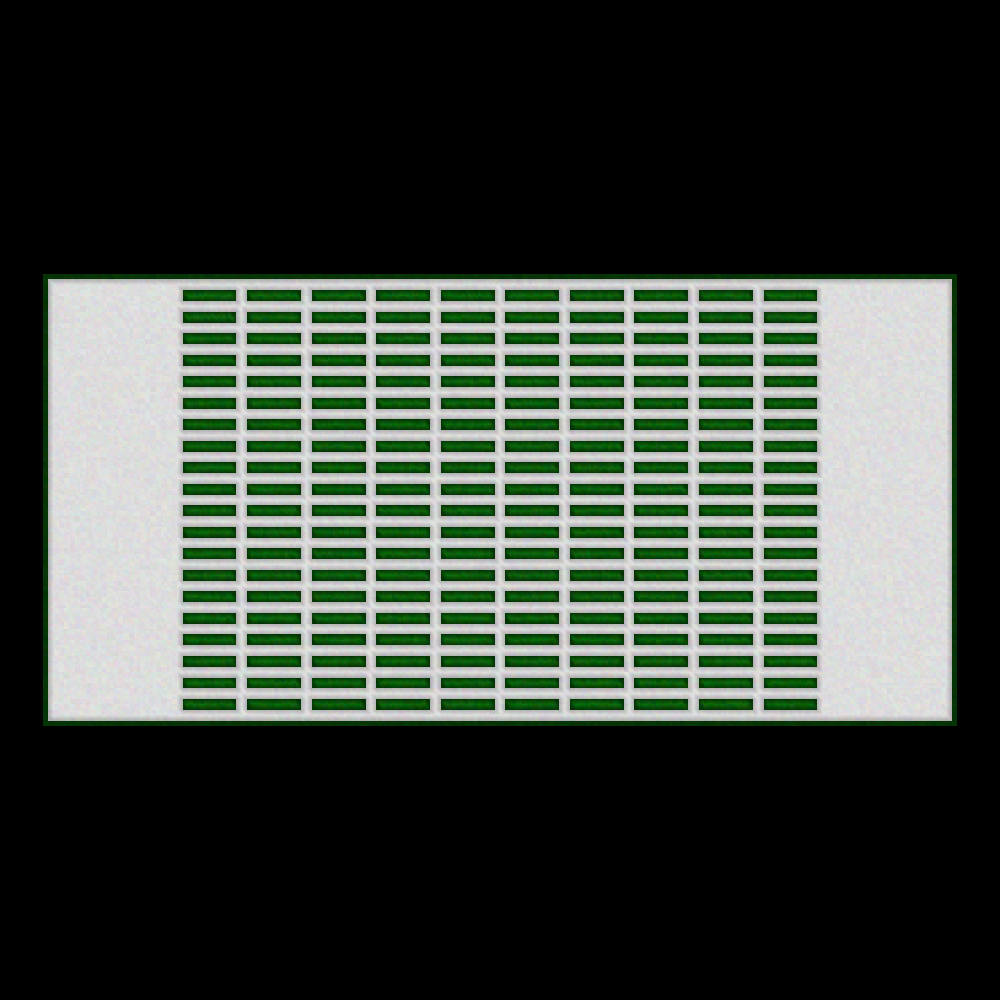}};

\node[inner sep=0pt] at (4.0,-13.02) {\small super-dense roadmap};
\node[inner sep=0pt] at (6.4,-13.68)
    {\includegraphics[width=1.72cm]{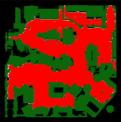}};

\end{tikzpicture}
\caption{Success rates on 32-neighbor grids and roadmaps.}
\label{fig:success}
\end{figure}
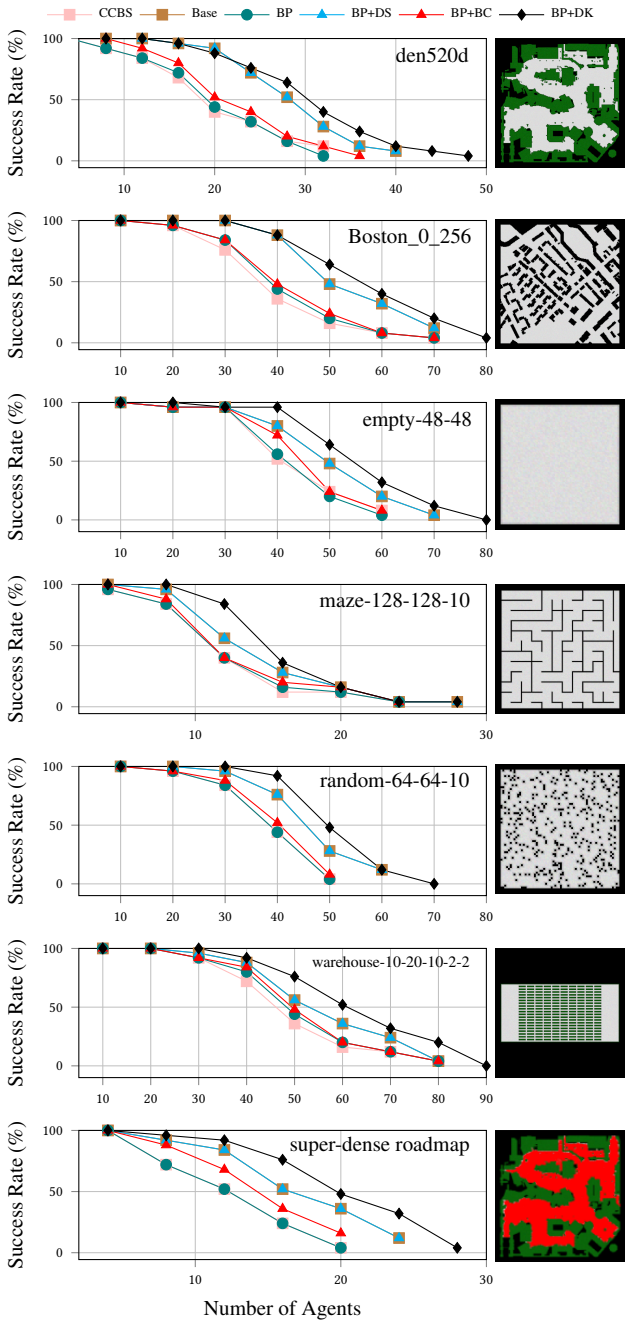

\begin{table}[h!]
\caption{Total problems solved in under 30 seconds on 4-neighbor grid MAPF benchmarks}
\label{tab:closed}
\centering
\resizebox{0.9\columnwidth}{!}{
\begin{threeparttable}
\begin{tabular}{l||rrrr}
 Map & Base & BP & DK & BP+DK \\
\specialrule{1pt}{0pt}{0pt}

Berlin\_1\_256       & 1,200 & \ul{\textbf{1,334}} & 1,200 & \ul{\textbf{1,334}} \\

Boston\_0\_256       & \textbf{1,112} & \textbf{1,140} & \textbf{1,112} & \ul{\textbf{1,152}} \\
Paris\_1\_256        & 1,432 & \ul{\textbf{1,524}} & 1,434 & \ul{\textbf{1,524}} \\
\specialrule{.2pt}{0pt}{0pt}
City Total           & 3,744 & \textbf{3,998} & 3,746 & \ul{\textbf{4,010}} \\
\midrule
den520d              &   \textbf{844} &   \textbf{860} &   \textbf{844} &   \ul{\textbf{868}} \\
brc202d              &   580 &   \textbf{604} &   586 &   \ul{\textbf{606}} \\
den312d              &   474 &   \textbf{510} &   474 &   \ul{\textbf{552}} \\
lak303d              &   474 &   \textbf{538} &   476 &   \ul{\textbf{552}} \\
orz900d              &   \ul{\textbf{626}} &   566 &   \textbf{622} &   564\\
ost003d              &   \textbf{562} &  \ul{\textbf{582}} &   \textbf{564} &   \ul{\textbf{582}} \\
\specialrule{.2pt}{0pt}{0pt}
DAO Total            & 3,560 & \textbf{3,660} & 3,566 & \ul{\textbf{3,724}} \\
\midrule
empty-8-8            &   312 &   316 &   314 &   \ul{\textbf{338}} \\
empty-16-16          &   \textbf{520} &   \textbf{520} &   \textbf{520} &   \ul{\textbf{546}} \\
empty-32-32          &   886 &   \ul{\textbf{920}} &   886 &   \textbf{918}\\
empty-48-48          & 1,058 & \textbf{1,110} & 1,058 & \ul{\textbf{1,144}}\\
\specialrule{.2pt}{0pt}{0pt}
Empty Total          & \textbf{2,776} & \textbf{2,866} & \textbf{2,778} & \ul{\textbf{2,946}} \\
\midrule
lt\_gallowstemplar   &   \textbf{632} &   \textbf{654} &   \textbf{634} &   \ul{\textbf{658}} \\
ht\_chantry          &   \textbf{570} &   \textbf{588} &   566 &   \ul{\textbf{592}}\\
ht\_mansion\_n       &   680 &   \textbf{726} &   680 &   \ul{\textbf{746}} \\
w\_woundedcoast      &   664 &   \textbf{6998} &   670 &   \ul{\textbf{706}} \\
\specialrule{.2pt}{0pt}{0pt}
DAO2 Total                & 2,546 & \textbf{2,666} & 2,550 & \ul{\textbf{2,702}} \\
\midrule
maze-32-32-2         &   \textbf{280} &   \textbf{278} &   \ul{\textbf{286}} &   \textbf{278}\\
maze-32-32-4         &   252 &   \ul{\textbf{256}} &   248 &   \textbf{254}\\
maze-128-128-2       &   \textbf{248} &   \textbf{254} &   \textbf{252} &   \ul{\textbf{256}}\\
maze-128-128-10      &   340 &   \ul{\textbf{364}} &   340 &   \textbf{360}\\
\specialrule{.2pt}{0pt}{0pt}
Maze Total            & 1,116 & \ul{\textbf{1,152}} & \textbf{1,126} & \textbf{1,148} \\
\midrule
random-64-64-10      & \textbf{1,210} & \textbf{1,222} & \textbf{1,210} & \ul{\textbf{1,266}} \\
random-64-64-20      &   \textbf{862} &   \textbf{860} &   \textbf{852} &   \ul{\textbf{920}}\\
random-32-32-10      &   614 &   \ul{\textbf{674}} &   614 &   \ul{\textbf{674}}\\
random-32-32-20      &   424 &   \textbf{438} &   424 &   \ul{\textbf{448}}\\
\specialrule{.2pt}{0pt}{0pt}
Random Total         & \textbf{3,110} & \textbf{3,194} & \textbf{3,100} & \ul{\textbf{3,308}} \\
\midrule
room-64-64-16        &   \textbf{424} &   \ul{\textbf{432}} &   \textbf{424} &   \textbf{428}\\
room-64-64-8         &   290 &   288 &   290 &   \ul{\textbf{314}}\\
room-32-32-4         &   286 &   \textbf{298} &   286 &   \ul{\textbf{306}}\\
\specialrule{.2pt}{0pt}{0pt}
Room Total             & \textbf{1,000} & \textbf{1,018} & \textbf{1,000} & \ul{\textbf{1,048}} \\
\midrule
w-10-20-10-2-2 & 1,124 & \textbf{1,372} & 1,124 & \ul{\textbf{1,390}} \\
w-10-20-10-2-1 & 1,022 & \ul{\textbf{1,078}} & 1,022 & \textbf{1,066} \\
w-20-40-10-2-2 & 1,890 & \ul{\textbf{1,998}} & 1,890 & \ul{\textbf{1,998}} \\
w-20-40-10-2-1 & 1,740 & 1,870 & 1,870 & \ul{\textbf{2,038}} \\

\specialrule{.2pt}{0pt}{0pt}

Warehouse Total           & 5,776 & \textbf{6,318} & 5,776 & \ul{\textbf{6,492}} \\
\specialrule{2pt}{0pt}{0pt}

Total & 23,628 & \textbf{24,872} & 23,642 & \ul{\textbf{25,378}} \\

\end{tabular}
\end{threeparttable}
}
\end{table}

\begin{table}[t!]
\caption{Summary of problems solved in under 30 seconds on 8-, 16- and 32-neighbor grid MAPF benchmarks}
\label{tab:summary}
\centering
\resizebox{0.9\columnwidth}{!}{
\begin{threeparttable}
\begin{tabular}{l||rrrr}
\multicolumn{1}{c||}{Map Type} & Base & BP & DK & BP+DK \\
\toprule[1.5pt]

 & \multicolumn{4}{c}{\textbf{8-neighbor grids}} \\
\specialrule{1.5pt}{0pt}{0pt}
city                 & 4,820 & \textbf{5,016} & 4,982 & \ul{\textbf{5,084}} \\
DAO                  & 4,298 & \textbf{4,394} & \textbf{4,382} & \ul{\textbf{4,474}} \\
empty                & 3,304 & \textbf{3,560} & 3,396 & \ul{\textbf{3,628}} \\
DAO2                 & 3,174 & \textbf{3,318} & 3,242 & \ul{\textbf{3,380}} \\
maze                 & \textbf{1,304} & \textbf{1,292} & \textbf{1,298} & \ul{\textbf{1,336}} \\
random               & 3,486 & \textbf{3,628} & 3,514 & \ul{\textbf{3,760}} \\
room                 & \textbf{1,142} & \textbf{1,148} & \textbf{1,124} & \ul{\textbf{1,168}} \\
warehouse            & 6,688 & \textbf{6,990} & \textbf{6,722} & \ul{\textbf{7,280}} \\
\specialrule{1.0pt}{0pt}{0pt}
Total & 28,216 & \textbf{29,346} & 28,660 & \ul{\textbf{30,110}} \\
\midrule

 & \multicolumn{4}{c}{\textbf{16-neighbor grids}} \\
\specialrule{1.5pt}{0pt}{0pt}
city                 & 3,906 & \textbf{4,044} & 3,958 & \ul{\textbf{4,148}} \\
DAO                  & 3,578 & 3,616 & 3,602 & \ul{\textbf{3,800}} \\
empty                & 3,252 & \textbf{3,326} & 3,278 & \ul{\textbf{3,366}} \\
DAO2                 & 2,598 & \textbf{2,662} & \textbf{2,624} & \ul{\textbf{2,776}} \\
maze                 & \textbf{1,142} & \textbf{1,154} & \textbf{1,146} & \ul{\textbf{1,188}} \\
random               & 3,080 & \textbf{3,124} & \textbf{3,110} & \ul{\textbf{3,260}} \\
room                 & \textbf{1,044} & \textbf{1,058} & \textbf{1,040} & \ul{\textbf{1,082}} \\
warehouse            & 6,516 & \textbf{6,650} & 6,550 & \ul{\textbf{6,926}} \\
\specialrule{1pt}{0pt}{0pt}
Total & 25,116 & \textbf{25,634} & 25,308 & \ul{\textbf{26,546}} \\
\midrule

 & \multicolumn{4}{c}{\textbf{32-neighbor grids}} \\
\specialrule{1.5pt}{0pt}{0pt}
city                 & 3,310 & 3,358 & 3,414 & \ul{\textbf{3,680}} \\
DAO                  & 2,860 & 2,884 & \textbf{2,940} & \ul{\textbf{3,058}} \\
empty                & 2,798 & 2,058 & 2,872 & \ul{\textbf{3,116}} \\
DAO2                 & 2,190 & 2,218 & 2,224 & \ul{\textbf{2,342}} \\
maze                 & 1,030 & 1,028 & 1,060 & \ul{\textbf{1,124}} \\
random               & 2,840 & 2,866 & 2,940 & \ul{\textbf{3,186}} \\
room                 & 1,032 & 1,042 & 1,040 & \ul{\textbf{1,106}} \\
warehouse            & 5,832 & 6,128 & 5,926 & \ul{\textbf{6,478}} \\

\specialrule{1pt}{0pt}{0pt}
Total & 21,892 & 21,582 & 22,416 & \ul{\textbf{24,090}}\\

\end{tabular}
\end{threeparttable}
}
\end{table}

\begin{table}[t!]
\caption{Total problems solved and mean runtime on roadmaps}
\label{tab:roadmaps}
\centering
\resizebox{0.9\columnwidth}{!}{
\begin{threeparttable}
\begin{tabular}{l||rrrr}
\multicolumn{1}{c||}{Map}  
 & Base & BP & DK & BP+DK \\
\toprule[1.5pt]
sparse               &   \textbf{434$\pm$8} &   396$\pm$8 &   \textbf{440$\pm$8} &   \ul{\textbf{444$\pm$8}} \\
dense                &   604$\pm$10 &   604$\pm$10 &   630$\pm$12 &   \ul{\textbf{712$\pm$12}} \\
super-dense          &   402$\pm$10 &   402$\pm$10 &   442$\pm$10 &   \ul{\textbf{474$\pm$11}} \\
\specialrule{1pt}{0pt}{0pt}
Total              & 1,440 &1,440 & 1,512 & \ul{\textbf{1,630}} \\
\end{tabular}
\end{threeparttable}
}
\end{table}

We now analyze the enhancements described in the previous section, namely: bypass, which is new to CCBS in this paper, biclique constraints (BC) which is newly formulated in this paper, and disjoint k-partite cliques (DK) which is new. All tests in this section were performed single-threaded, on cloud compute instances that report an Intel Xeon 2.5GHz processor. In addition to the three roadmaps: ``sparse'', ``dense'' and ``super-dense'' and the four grid maps focused on by the original CCBS authors, we test our enhancements in all 44 of the MAPF grid benchmarks~\cite{stern2019multi} with $2^k$ neighborhood~\cite{neighborhoods} connectivities, namely 4-, 8-, 16- and 32-neighborhoods. Agents are circular, with a radius of $\sqrt{2}/4$. All tests were run by starting with 5 agents and incrementing the number of agents by 2 until the problem instance became unsolvable in under 30 seconds. We thank the original CCBS authors for making their code publicly available. Our implementation is based on theirs and is also freely available\footnote{https://github.com/thaynewalker/CCBS}. We updated some of the memory management, which made it up to 60\% faster than the  original. All of our results compare against our enhanced code. For this reason, some of our baseline results differ from previously published ones.

We will describe each of our experiments in turn and then discuss them all together. Figure \ref{fig:success} shows the success rates; the percentage of problem instances solvable in under 30 seconds for increasing numbers of agents. The rates are computed over 25 problem instances for each map. All plots, except for the super-dense roadmap are on 32-neighbor grids. We remind the reader that results are optimal in terms of cost (i.e., shortest non-conflicting paths), and that adding a single agent to a problem instance represents an exponential increase in the problem instance's computational complexity (it is exponential in the number of agents~\cite{yu2013multi}). Hence, even a small gain (e.g., a few percentage points) can mean that a multiplicative reduction in work is actually realized.

Tables \ref{tab:closed}, \ref{tab:summary} and \ref{tab:roadmaps} show the sum total of the maximum number of agents solvable over 25 problem instances in under 30 seconds per map. The statistics with the best result underlined and those within the 95th percentile of the best are in bold. The label ``CCBS'' in Figure \ref{fig:success} is CCBS with all enhancements from the original authors except DS. The label ``Base'' is CCBS with all enhancements, including DS, the previous state of the art. In all of the tables, the column labeled ``Base'' is also the previous state of the art.

Table \ref{tab:closed} shows totals for all grid maps for 4-neighbor grids, totals for each class of maps, and an overall total. 4-neighbor grids in this context are similar to those for ``classic'' MAPF, except instead of fixed wait actions, agents may wait an arbitrary amount of time. Table \ref{tab:summary} shows aggregate group results for the same groupings as Table \ref{tab:closed}, and the grand total for all remaining connectivity settings, namely 8-, 16- and 32-connected grids. Table \ref{tab:roadmaps} shows the results for all settings on probabilistic roadmaps~\cite{kavraki1996probabilistic}. ``sparse'' contains 158 nodes and 349 edges with a mean vertex degree of 4.2, ``dense'' contains 878 nodes and 7,341 edges with a mean degree of 16.7, and ``super-dense'' contains 11,342 vertices and 263,533 edges with a mean degree of 100.4.

\subsection{Discussion}

In Figure \ref{fig:success}, the BP enhancement (teal circle) success rate is not significantly better or worse than CCBS alone (brown square). It never performs worse than CCBS in any of the 32-neighbor maps. On the other hand, in Table \ref{tab:closed}, which is on 4-neighbor maps, we see that adding BP to Base (previous state-of-the-art) yields statistically significant gains in nearly every map. Table 2 shows a similar trend for all 8-neighbor grids, but as the connectivity is increased to 16- and 32-neighbor grids, the improvement from BP becomes less significant. The cost symmetries in these settings decreases as the connectivity increases. Ultimately, in the road maps, which have no cost symmetries, we see no improvement with BP, nor any significant decline in performance. From this we learn that (1) BP offers significant performance improvements when there are many cost symmetries in the graph, and that the performance benefits are directly proportional to the amount of cost symmetries. (2) The cost of BP is not significant, even with no symmetries in the map. (3) BP is complimentary to DS and BC.

The trend in Figure \ref{fig:success} shows that BC provides a consistent improvement over CCBS in 32-neighbor grids. It is complimentary to DS, as evidenced by the fact that BP+DK performs better than both BP+DS and BP+BC. The trend in all tables shows that the performance improvement of DK is correlated to the mean vertex degree in the graph. The amount of improvement is especially significant in the dense and super-dense  road maps, where BP offers no benefit. In these roadmaps, the branching factor is high, with many edge crossings, a situation which is conducive to large bicliques. Still, the cost of DK is not significant in planar graphs. From this we learn that (1) DK offers significant performance improvements when many edges cross in the graph. (2) BP and DK are complimentary, as evidenced by BP being stronger in settings with many cost symmetries and DK being stronger in settings with few cost symmetries.

With CCBS, because an agent may have multiple different wait actions at a location, BCGs larger than 1x1 are possible, thus a small benefit over Base is shown with BC in Table \ref{tab:closed}. 

Because the BP enhancement is never significantly detrimental to performance, and provides a benefit even when there are few cost symmetries, it can be used generally. The DK enhancement tends to work best in many cases where BP does not, and it also has no significant execution cost, hence it can be used generally.

Finally, combining BP with DK consistently beats state-of-the-art by statistically significant margins. Compared to the previous state-of-the-art, in super-dense roadmaps our enhancements allow solutions for up to 10\% more agents and in grid maps and for up to 20\% more agents in super dense roadmaps.

\section{Conclusion and Future Work}
We have formulated and tested novel enhancements for CCBS and tested performance in roadmaps and on the full set of MAPF benchmarks on various connectivity settings. We found that disjoint splitting, bypass and biclique constraints are are complimentary and that using them together allows a statistically significant improvement over state-of-the-art generally for all MAPF benchmarks and sparse to super dense graphs. Bypassing is most effective in graphs with cost symmetries and biclique constraints are most effective in graphs with few cost symmetries and densely crossing edges.

\section*{Acknowledgments}
The research at the University of Denver was supported by the National Science Foundation (NSF) grant number 1815660 and Lockheed Martin Corp. 
Research at the university of Alberta was funded by the Canada CIFAR AI Chairs Program. We acknowledge the support of the Natural Sciences and Engineering Research Council of Canada (NSERC). Research at Ben Gurion University was supported by BSF grant number 2017692.

\bibliographystyle{ACM-Reference-Format}
\bibliography{main}

\end{document}